\documentclass[11pt]{article}

\usepackage[ruled]{algorithm2e}
\usepackage{algorithmic}
\usepackage{amsfonts}
\usepackage{amsmath}
\usepackage{amssymb}
\usepackage{amsthm}
\usepackage{fullpage}
\usepackage{graphicx}
\usepackage{natbib}
\usepackage{tikz}
\usepackage{times}

\usepackage{mdframed}

\SetEndCharOfAlgoLine{}
\SetArgSty{}

\mdfsetup{frametitlealignment=\centering}
\mdfdefinestyle{offset}{backgroundcolor=white,linecolor=black,innerrightmargin=15pt,innermargin=10pt,outermargin=10pt,innertopmargin=.5\baselineskip,innerbottommargin=.5\baselineskip}

\newtheorem{theorem}{Theorem}[section]
\newtheorem{lemma}[theorem]{Lemma}

\newtheorem{corollary}[theorem]{Corollary}

\newenvironment{reminder}[1]{\smallskip
	\noindent {\bf Reminder of #1 }\em}{}

\newcommand{\BalanceSubproblem}{USM Balance Subproblem}
\newcommand{\OnlineUSM}{Online USM}

\newcommand{\E}{\mathbb{E}}
\newcommand{\cov}{\text{cov}}

\newcommand{\Yes}[1][i]{Z^+_{#1}}
\newcommand{\No}[1][i]{Z^-_{#1}}

\newcommand{\RewardAlg}{R_{alg}}
\newcommand{\CostYes}{C_{yes}}
\newcommand{\CostNo}{C_{no}}
\newcommand{\Potential}{\Phi}
\newcommand{\PotentialAlg}{\Potential_{alg}}
\newcommand{\PotentialYes}{\Potential_{yes}}
\newcommand{\PotentialNo}{\Potential_{no}}

\newcommand{\RR}{\mathbb{R}}
\newcommand{\sse}{\subseteq}
\newcommand{\nar}{no-$\alpha$-regret\xspace}

\newcommand{\prob}[2][]{\mathbf{Pr}\ifthenelse{\not\equal{}{#1}}{_{#1}}{}\!\left[#2\right]}
\newcommand{\expect}[2][]{\mathbf{E}\ifthenelse{\not\equal{}{#1}}{_{#1}}{}\!\left[#2\right]}

\title{An Optimal Algorithm for Online Unconstrained Submodular Maximization}

\author{
 	Tim Roughgarden\thanks{Department of Computer Science,
 		Stanford University, 474 Gates Building, 353 Serra Mall, Stanford, CA 94305.
 		This research was supported in part by NSF grant
 		CCF-1524062 and a Google Faculty Research Award.
 		Email: {\tt tim@cs.stanford.edu}.} \and Joshua R. Wang\thanks{Department of Computer Science,
 		Stanford University, 460 Gates Building, 353 Serra Mall, Stanford, CA 94305.
 		This research was supported in part by a Stanford Graduate Fellowship and NSF grant
                                  CCF-1524062.
 		Email: {\tt joshua.wang@cs.stanford.edu}.} 
 }

\begin{document}

\maketitle

\begin{abstract}
  We consider a basic problem at the interface of two fundamental
  fields: {\em submodular optimization} and {\em online learning}.  In
  the {\em online unconstrained submodular maximization (online USM)
  problem}, there is a universe $[n]=\{1,2,\ldots,n\}$ and a
  sequence of $T$ nonnegative (not necessarily monotone) submodular functions arrive over time.  The
  goal is to design a computationally efficient online algorithm,
  which chooses a subset of $[n]$ at each time step as a function only
  of the past, such that the accumulated value of the chosen subsets
  is as close as possible to the maximum total value of a fixed subset
  in hindsight.  Our main result is a polynomial-time 
  no-$\tfrac12$-regret algorithm for this
  problem, meaning that for every sequence of nonnegative
  submodular functions, the algorithm's expected total value is at least
  $\tfrac12$ times that of the best subset in hindsight, up to an
  error term sublinear in $T$.
  The factor of $\tfrac12$ cannot be improved upon by
  any polynomial-time online algorithm when the submodular functions are presented
  as value oracles.
  Previous work on the offline problem implies that picking a subset uniformly at
  random in each time step achieves zero $\tfrac14$-regret.

  A byproduct of our techniques is an explicit subroutine for the two-experts
  problem that has an unusually strong regret guarantee: the total value
  of its choices is comparable to twice the total value of either
  expert on rounds it did not pick that expert. This subroutine may be
  of independent interest.
\end{abstract}


\section{Introduction}

The problem we study, {\em online unconstrained submodular maximization (online
USM)}, lies in the intersection of two fundamental fields: {\em submodular
optimization} and {\em online learning}.

\paragraph{Submodular optimization.}
A nonnegative real-valued set function~$f:2^{[n]} \rightarrow \RR_+$ 
defined on the ground set $[n] = \{1,2,\ldots,n\}$
is {\em submodular}
if it exhibits diminishing returns, in the sense that
\[
f(S \cup \{i\}) - f(S) \le f(T \cup \{i\})-f(T)
\]
whenever $T \subseteq S$ and $i \notin S$.\footnote{Note that $f$ is
  not assumed to be monotone.}  Submodular functions can be used to
model a wide array of important problems, and for this reason have
been extensively studied for decades in theoretical computer science
(e.g.~\cite{shaddin}), combinatorial optimization
(e.g.~\cite{vondrak}), economics (e.g.~\cite{milgrom}), and machine
learning (e.g.~\cite{bach13}).  Perhaps the most basic problems in
submodular optimization are to minimize or maximize a submodular
function (without constraints).  While the former problem admits
(highly non-trivial) polynomial-time
algorithms~\citep{GLS88,IFF01,S00}, unconstrained maximization is hard
to approximate better than a factor of~$\tfrac 12$ in polynomial time
\citep{FMV11,DV12}.  Indeed, many fundamental $NP$-hard problems are
special cases of unconstrained submodular maximization (USM),
including undirected and directed versions of graph and hypergraph cut
problems (e.g.~\cite{GW95,HZ01}), maximum facility location problems
(e.g.~\cite{AS99}), and certain restricted satisfiability problems
(e.g.~\cite{GK05}).  Also, approximation algorithms for the USM
problem have been used as subroutines in many other algorithms,
including those for social network marketing \citep{HMS08}, market
expansion \citep{revsub}, and the computation of the least core value
in a cooperative game \citep{SU13}.

\paragraph{Online learning.}
The goal in online learning is to make good decisions over time with
knowledge only of the past.  In the standard ``experts'' setup,
there
is a known set $A$ of actions and a time horizon $T$.  At each time
step $t=1,2,\ldots,T$, the online algorithm has to first choose a
action $a^t \in A$, and an adversary subsequently chooses a reward
vector $r^t:A \rightarrow [0,1]$.
Given a history of actions $a^1,\ldots,a^T$
and reward vectors $r^1,\ldots,r^T$,
the {\em regret} of the algorithm is the
difference between the maximum total reward $\max_{a \in A}
\sum_{t=1}^T r^t(a)$ of a fixed action in hindsight
and the total reward $\sum_{t=1}^T r^t(a^t)$ earned by the algorithm.
The goal in online learning is to design algorithms with expected
regret $o(T)$ as $T \rightarrow \infty$.

Ignoring computational issues, this goal is well understood:
there are randomized algorithms (like ``Follow the Perturbed Leader''
\citep{KV05} and ``Multiplicative Weights'' \citep{CBMS07,FS97}) 
with worst-case expected regret $O(\sqrt{T \log
  |A|})$, and no algorithm can do better (see
e.g.~\cite{CL06}).  
However, the generic algorithms that achieve this regret bound require
computation at least linear in $|A|$ at each time step.  Thus, when the action
space $A$ has size exponential in the parameters of interest, these algorithms
are not computationally efficient.

\paragraph{Online USM.}
We consider the natural online learning version of the USM problem.
There is a universe~$[n]=\{1,2,\ldots,n\}$, known in advance.  
Actions correspond to subsets of the universe,
and submodular functions arrive online.
\begin{itemize}
\item At each time step $t=1,2,\ldots,T$:
\begin{itemize}
\item The algorithm picks a probability distribution $p^t$ over subsets of $[n]$.
\item An adversary picks a submodular function $f^t:2^{[n]} \rightarrow
  [0,1]$.
\item A subset $S^t$ is chosen according to the distribution $p^t$,
  and the algorithm reaps a reward of $f^t(S^t)$.
\item The adversary reveals $f^t$ to the algorithm.
\end{itemize}
\end{itemize}

The goal is to design a computationally efficient online algorithm
with worst-case expected regret as small as possible.
Applying the generic no-regret algorithms to this problem requires
per-step computation exponential in $n$.
Indeed, unless $RP=NP$, {\em there does not exist a polynomial-time
  no-regret algorithm for the online USM problem}.\footnote{Standard
  arguments show that any
polynomial-time (randomized) \nar algorithm for online USM
yields a polynomial-time randomized
$(\alpha+\epsilon)$-approximation algorithm for the offline USM
problem for every constant $\epsilon > 0$.  The basic idea is to feed
the offline input~$f$ into the online algorithm over and over again,
and return the best of the subsets output by the online
algorithm.  
Since \cite{DV12} prove that offline
USM is hard to approximate to within a factor better than $\tfrac
12$ (assuming $NP \neq RP$), the same lower bound carries over to the
online version of the problem.} 
This negative result motivates
following in the footsteps of~\cite{KKL09} and defining, for $\alpha
\in [0,1]$,
the {\em $\alpha$-regret} of an algorithm (w.r.t.\ actions
$S^1,\ldots,S^T$ and functions $f^1,\ldots,f^T$) as
the difference between 
$\alpha$ times the cumulative reward of the best fixed action in hindsight
and that earned by the algorithm:
\begin{equation}\label{eq:alpha}
\alpha \cdot \max_{S \sse [n]} \sum_{t=1}^T f^t(S)
- \sum_{t=1}^T f^t(S^t).
\end{equation}
A {\em no-$\alpha$-regret algorithm} is one whose worst-case
expected $\alpha$-regret is bounded by $O(T^c)$ for some constant $c < 1$
(with the big-O suppressing any dependence on $n$).
The worst case is taken over the adversary's choice of functions, and
the expectation is over the coin flips of the algorithm.
A basic question is:
\begin{itemize}

\item [] {\em What is the largest constant $\alpha$ such that there exists a
computationally  efficient no-$\alpha$-regret algorithm for online USM?}

\end{itemize}
By ``efficient,'' we mean that the number of operations performed by
the algorithm in each time step is bounded by some polynomial function
of~$n$, the size of the universe.\footnote{Unless otherwise noted, we
  assume that each submodular function $f$ in the input: (i) has
  description length polynomial in $n$; and (ii) given a subset $S
  \sse [n]$, the value $f(S)$ of~$f$ can be evaluated in
  time polynomial in $n$.  All of our results also hold in the ``value
  oracle'' model, with submodular functions given as ``black boxes''
  that support value queries.  Here, our online algorithm uses only
  polynomially many (in $n$) value queries and polynomial additional
  computation.  The lower bound continues to apply and becomes
  unconditional in the value oracle model (following \cite{FMV11}).}

Our main result is a tight answer to this question: 
\begin{itemize}

\item [] {\em $\alpha =
\tfrac{1}{2}$ is achievable, and no larger value of $\alpha$ can be
achieved (unless $RP = NP$).}

\end{itemize}
Prior to our work, the best result known (which follows from
\cite{FMV11}) was that $\alpha = \tfrac 14$ can be achieved by picking a
subset uniformly at random in every time step.

\paragraph{Offline-to-online reductions.}
Our results also contribute to the burgeoning line of work on
``offline-to-online reductions.''  Here, the question is whether or
not an efficient $\alpha$-approximate oracle for the offline version of a
problem (i.e., computing the best strategy in hindsight, given a
sequence of implicitly defined reward vectors) can be
translated in ``black-box'' fashion to an efficient \nar online
algorithm. 

The existing offline-to-online reductions
apply only to linear online optimization
problems~\citep{AK08,FHT13,KV05,KKL09} or require
an exact best-response oracle~\citep{D+17,Z03}, and thus do not apply
to the USM problem.
Meanwhile, \cite{HK16} prove that there is no fully
general black-box reduction: there exists a (somewhat artificial)
problem such that, even with an exact oracle for the offline
version of the problem, achieving sublinear regret requires
a super-polynomial amount of computation.  Thus for some problems,
there is a fundamental difference between what is possible offline
versus online.  It remains an open question whether or not there is a
``natural'' optimization problem with a provable separation between
its offline and online versions.

Online USM is arguably one of the most natural online problems where
the state-of-the-art is silent on whether or not there are online
guarantees matching what is possible offline, and this paper resolves
this question (in the positive).

\subsection{Related Work}

\cite{FMV11} were the first to
rigorously study the general USM problem. They showed that a uniformly
random subset $S$ achieves a $\frac14$-approximation (in
expectation). They also provided an algorithm, based on noisy local
search, with an approximation
guarantee of $\frac25$.
Finally, they proved that in the value oracle model, achieving an
approximation of $\frac12 + \epsilon$ requires an exponential number
of queries in the worst case. The noisy local search technique was
improved slightly by 
\cite{GV11} and further by 
\cite{FNS11}. A breakthrough occurred
when 
\cite{BFNS15} showed that a simple strategy
could be used to achieve a (tight) $\frac12$ approximation
ratio. Their algorithm was randomized, but was later derandomized by
\cite{BF16}.
The initial lower bound in \cite{FMV11}
was generalized by
\cite{DV12}, who proved the same bound even for succinctly
represented functions (polynomial description and evaluation time),
conditioned on $RP \neq NP$.

Online submodular {\em minimization} is considered by 
\cite{HK12}.  Here, the offline problem can be solved exactly
with a polynomial number of value queries (e.g.~\cite{GLS88}), and the
main result in~\cite{HK12} is an efficient no-regret algorithm for the
online setting.  Some extensions to online submodular minimization
with constraints are given by 
\cite{JB11}.
\cite{SG09} considered a fairly general online
submodular maximization problem. In particular, their problem captures
the online problem where the algorithm receives a series of monotone
submodular functions and wants to maximize them subject to a knapsack
constraint.

Finally, 
\cite{BFS15} study a problem that they call
``online submodular maximization,'' but where there is only a single
function and the elements of the universe arrive over time.  This
version of the problem is in the tradition of competitive online
algorithms rather than no-regret learning algorithms, and hence is
quite different from the online USM problem that we study.

\subsection{Our Techniques}

We now provide an overview of the main ideas used in obtaining a no-$1/2$-regret algorithm for \OnlineUSM{}. The overall argument is divided into two phases. In the first phase, whose main result is captured in Theorem~\ref{thm:balance-subproblem}, we propose a general class of algorithms for the \OnlineUSM{} algorithm, based on the Buchbinder et al. analysis for the offline problem~\citep{BFNS15}. This class is parameterized by our choice of subroutine, and the main result of this phase states that the performance of our algorithm with respect to \OnlineUSM{} is precisely characterized by the performance of its subroutine with respect to a specific task: the \BalanceSubproblem{}. Stopping here already yields a novel result: using a no-regret algorithm for the (two) experts problem, such as Multiplicative Weights (see, e.g., \cite{CL06}), as our subroutine would give us a no-$1/3$-regret algorithm for \OnlineUSM{}. The previously best known is returning a uniform random point in every round, which is a no-$1/4$-regret algorithm.

In the second phase of our argument, we focus our efforts on designing
a good subroutine for the \BalanceSubproblem{}. The main result is
Theorem~\ref{thm:balancer} in which we prove that our proposed
subroutine satisfies the condition which results in a no-$1/2$-regret
algorithm for \OnlineUSM{}.\footnote{It is also possible to apply
  Blackwell's Approachability Theorem~\citep{Blackwell56} to get a
  subroutine which satisfies this condition as well.  (Thanks to
  anonymous reviewer for pointing this out.)  We nevertheless provide
  our own subroutine along with its proof, as this makes the entire
  online USM algorithm and its analysis more explicit.} Using any
algorithm with a no-regret guarantee for the (two) experts problem is
provably insufficient; for \emph{every} such algorithm the result is
an algorithm for \OnlineUSM{} (when using the aforementioned no-regret
algorithm as a subroutine) with \emph{linear} expected
$1/2$-regret. In other words, the binary-action task we are attempting
to solve really is distinct from the experts problem. Roughly
speaking, the situtation in the \BalanceSubproblem{} is as
follows. The algorithm wants to make progress, but at the same time
the adversary is advancing its own goals on two different fronts. The
problem is named after the need to balance the losses incurred on
these two fronts; an algorithm that focuses on the experts problem
only considers the total loss. To complicate matters further, one of
the possible actions may have a negative value, and choosing such an
action incurs two types of loss: it both sets back the algorithm while
advancing the adversary's agenda. One key technical contribution is a
potential-based analysis, which succinctly captures the relationship
between the algorithm's state and the status of the no-regret
guarantee we want to prove; much of the complexity is hidden in
identifying the appropriate potential functions.

Finally, we conclude by generalizing the analysis to work against adaptive adversaries. The main contribution here is a covariance-based argument which guards against the adaptive adversary blowing up the variance of our algorithm by choosing its future inputs to depend on the results of past coin flips.

\subsection{Organization}

The first phase of our proof is conducted in Section~\ref{s:framework}; we provide a framework for \OnlineUSM{}, identify the subproblem of interest, and give our main reduction. Our proposed subroutine and main result are stated in Section~\ref{s:half}, and the proof is carried out in Appendix~\ref{a:half}. Finally, in Appendix~\ref{a:adaptive}, we discuss the generalization to adaptive adversaries.

\section{An \OnlineUSM{} Framework}\label{s:framework}

We begin by presenting our framework for \OnlineUSM{}, which is based on the BFNS offline algorithms~\citep{BFNS15}.

In order to make the offline problem tractable, these algorithms transform the task of choosing a subset $S \subseteq [n]$, which has $2^n$ possible choices, into the $n$ tasks of choosing whether element $i$ should be in $S$ or not, each of which have just two possible choices. To be more specific, we begin with two candidate solutions: $X_0$ as the empty set and $Y_0$ as the entire universe. We then proceed in $n$ iterations. In iteration $i$, we want to make the two candidate solutions agree on element $i$. Hence we must either add $i$ to $X_{i-1}$ or remove $i$ from $Y_{i-1}$. To decide which, we compute the marginal values of these two options according to our function $f$. In particular, let:
\begin{align*}
  \alpha_i &= f(X_{i-1} \cup \{i\}) - f(X_{i-1}), \\
  \beta_i &= f(Y_{i-1} \setminus \{i\}) - f(Y_{i-1}).
\end{align*}

Roughly speaking, we want to favor the larger of these two values. Due to submodularity, $\alpha_i + \beta_i \ge 0$ always (since
$X_{i-1} \subseteq Y_{i-1}$; the two sets agree on all elements up to
$i-1$ after which $Y_{i-1}$ has everything and $X_{i-1}$ has
nothing). The 
analysis in \cite{BFNS15} 
shows that deterministically picking based on the larger value gives a $1/3$-approximation overall. However, randomly choosing to include $i$ with probability $\frac{\alpha_i}{\alpha_i + \beta_i}$ and to remove $i$ with probability $\frac{\beta_i}{\alpha_i + \beta_i}$ can\footnote{Only necessary when both $\alpha_i$ and $\beta_i$ are both positive. If only one value is positive, we need to always pick that choice.} improve this to a $\frac12$-approximation overall.

This suggests an online framework which uses specialized binary-action subroutines to make these smaller decisions. Algorithm~\ref{alg:online-usm-framework} implements this idea, using the after-the-fact marginal values of the most recent submodular function to provide feedback to its subroutines.

\begin{algorithm}
\caption{\OnlineUSM{} Framework.}
\label{alg:online-usm-framework}
\vspace{.5\baselineskip}

  \SetKwInOut{Input}{input}
  \SetKwInOut{Output}{output}
  \Input{Subroutine $A$ (binary-action), submodular functions $\{f^t: [n] \to [0, 1]\}_t$}
  \Output{Subsets $\{S^t \subseteq [n]\}_t$}
  Run $n$ copies of $A$: $A_1, \ldots, A_n$. \\
  \For{round $t = 1$ to $T$}{
    Initialize subset $X^t_0 \leftarrow \{\}$ and subset $Y^t_0 \leftarrow [n]$. \\
    \For{$i = 1$ to $n$}{
      Ask $A_i$ whether $i$ should be in $S^t$. \\
      If $A_i$ says yes, set $X^t_i \leftarrow X^t_{i-1} \cup \{i\}$ and $Y^t_i \leftarrow Y^t_{i-1}$. \\
      If $A_i$ says no, set $X^t_i \leftarrow X^t_{i-1}$ and $Y^t_i \leftarrow Y^t_{i-1} \setminus \{i\}$.
    }
    Output $X^t_n$ for round $t$, and receive as input the submodular function $f^t$. \\
    \For{$i = 1$ to $n$}{
      Let $\alpha^t_i \leftarrow f^t(X^t_{i-1} \cup \{i\}) - f^t(X^t_{i-1})$. \\
      Let $\beta^t_i \leftarrow f^t(Y^t_{i-1} \setminus \{i\}) - f^t(Y^t_{i-1})$. \\
      Report $(\alpha^t_i, \beta^t_i)$ to $A_i$ as the rewards for yes and no, respectively.
    }
  }
\end{algorithm}

What guarantees do we need on the subroutine in our framework to get a no-regret guarantee for \OnlineUSM{}? We present the necessary guarantees as another online problem, which we call the \BalanceSubproblem{}.

\subsection{The \BalanceSubproblem{}}

The \BalanceSubproblem{} is a binary-action online problem. In each
round $t$, the algorithm chooses ``yes'' or ``no'' and then the
adversary reveals a point $(\alpha^t, \beta^t)$. Based on the
algorithm's decision and the adversary's point, three quantities are
updated. The algorithm has a total accumulated reward, denoted
$\RewardAlg$. The adversary accumulates two \emph{separate} piles of
missed opportunities, which will be denoted $\CostYes$ and $\CostNo$.

The adversary's point $(\alpha^t, \beta^t)$ lies in $\mathbb{R}^2$ subject to three constraints:
\begin{itemize}
  \item $-1 \le \alpha^t \le +1$,
  \item $-1 \le \beta^t \le +1$, and
  \item $\alpha^t + \beta^t \ge 0$.
\end{itemize}
The allowed space of points is illustrated in Figure~\ref{fig:three-moves}. When the algorithm chooses yes, $\RewardAlg$ increases by $\frac12 \alpha^t$ and $\CostNo$ increases by $\beta^t$. If it instead chooses no, then $\RewardAlg$ increases by $\frac12 \beta^t$ and $\CostYes$ increases by $\alpha^t$.\footnote{Why is the algorithm reward seemingly half of what it should be? The heart of the matter is that because our \OnlineUSM{} analysis is keeping track of two candidate solutions, it winds up double-counting progress. We correct for this factor with our rewards.}

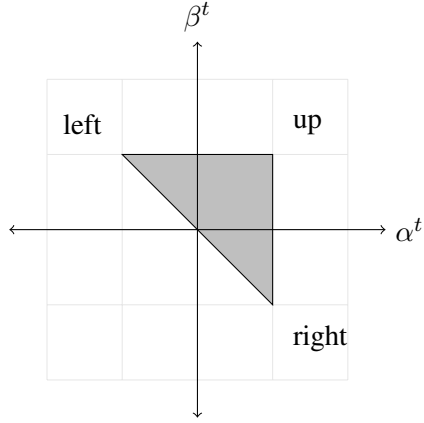
\begin{figure}
\centering
\begin{tikzpicture}[scale=1.0]
  \draw[step=1cm,black!10,very thin] (-2, -2) grid (2, 2);
  \draw[fill=black!25] (1, 1) -- (1, -1) -- (-1, 1) -- cycle;
  
  \draw[<->] (-2.5, 0) -- (2.5, 0) node[right] {$\alpha^t$};
  \draw[<->] (0, -2.5) -- (0, 2.5) node[above] {$\beta^t$};

  \node[label=above right:up] (u) at (1, 1) {};
  \node[label=above left:left] (l) at (-1, 1) {};
  \node[label=below right:right] (r) at (1, -1) {};
\end{tikzpicture}
\caption{The \BalanceSubproblem{} adversary's possible moves are convex combinations of up $(+1, +1)$, right $(+1, -1)$, and left $(-1, +1)$.}
\label{fig:three-moves}
\end{figure}

We say that the $\alpha$-regret of an algorithm for the \BalanceSubproblem{} is
\[
  \alpha \cdot \max \left( \CostYes, \CostNo \right) - \RewardAlg.
\]
As usual, we say that an algorithm has no-$\alpha$-regret if its worst-case expected $\alpha$-regret is bounded by $O(T^c)$ for some constant $c < 1$, with the big-O supressing any dependence on $n$. The worst case is still taken over the adversary's choice of points, and the expectation is over coin flips of the algorithm. If we have a no-$\alpha$-regret algorithm for the \BalanceSubproblem{}, its $\RewardAlg$ is comparable to the better of $\CostYes$ and $\CostNo$, in expectation.

We have been building up to the following theorem reducing \OnlineUSM{} to the \BalanceSubproblem{}:
\begin{theorem}\label{thm:balance-subproblem}
  For any constant $\alpha > 0$, when given a subroutine $A$ with $g(T)$ $\alpha$-regret for the \BalanceSubproblem{}, Algorithm~\ref{alg:online-usm-framework} has $O(n \cdot g(T))$ $\frac{\alpha}{1+\alpha}$-regret.
\end{theorem}

\begin{proof}
  In this proof, we use $\Yes$ to denote the rounds where the subroutine $A_i$ returned yes; $\No$, no.
  
  For the purposes of comparison, we also track the evolution of a \emph{third} set. For each round $t$, define $OPT^t_0$ to be offline optimal set (the best fixed set over all rounds, independent of $t$). Let $OPT^t_i = OPT^t_{i-1} \cup \{i\}$ if $t \in \Yes$ and $OPT^t_i = OPT^t_{i-1} \setminus \{i\}$ if $t \in \No$.  In English, $OPT^t_i$ begins (when $i=0$) at the optimal answer and has its entries changed to match the decisions made within round $t$ until it finishes (when $i = n$) at the algorithm's choice: $OPT^t_n = X^t_n = Y^t_n$. Intuitively, our algorithm will perform well if it manages to grow $f^t(X^t_i)$ and/or $f^t(Y^t_i)$ while lowering the value of $f^t(OPT^t_i)$ relatively little in comparison.

  Armed with these three evolving sets, we now want to know how the decisions of our subroutines impact their values. Suppose that in round $t$, the $i$th subroutine says yes. By construction, we know that:
  \begin{itemize}
    \item $f^t(X^t_i) = f^t(X^t_{i-1}) + \alpha^t_i$,
    \item $f^t(Y^t_i) = f^t(Y^t_{i-1})$,
    \item if element $i$ was in $OPT$, then $f^t(OPT^t_i) = f^t(OPT^t_{i-1})$ because $OPT^t_i = OPT^t_{i-1}$,
    \item if element $i$ was not in $OPT$, then $f^t(OPT^t_i) \ge f^t(OPT^t_{i-1}) - \beta^t_i$ by the submodularity of $f^t$, noting that $Y^t_{i-1}$ is a superset of $OPT^t_{i-1}$.
  \end{itemize}
  By the same reasoning, when the $i$th subroutine says no, all of the following happen:
  \begin{itemize}
    \item $f^t(X^t_i) = f^t(X^t_{i-1})$,
    \item $f^t(Y^t_i) = f^t(Y^t_{i-1}) + \beta^t_i$,
    \item if element $i$ was not in $OPT$, then $f^t(OPT^t_i) = f^t(OPT^t_{i-1})$, again because $OPT^t_i = OPT^t_{i-1}$, and
    \item if element $i$ was in $OPT$, then $f^t(OPT^t_i) \ge f^t(OPT^t_{i-1}) - \alpha^t_i$, again by submodularity of $f^t$, noting that $X^t_{i-1}$ is a subset of $OPT^t_{i-1}$.
  \end{itemize}
  
  Table~\ref{tab:usm-iopt} depicts these changes for the case where element $i$ is in $OPT$.
  \begin{table}
  \centering
  \begin{tabular}{ c | c | c }
    Option & Increase of $f^t(X^t_i) + f^t(Y^t_i)$ & Decrease of $f^t(OPT^t_i)$ \\ \hline
    Choosing $i$ & $\alpha^t_i$ & 0 \\
    Not Choosing $i$ & $\beta^t_i$ & $\le \alpha^t_i$ \\
  \end{tabular}
  \caption{How values change when $i \in OPT$.}
  \label{tab:usm-iopt}
  \end{table}
  
  Now, fix an element $i \in [n]$. Summing the first two bullets above (for both cases) over all the rounds, we have:
  \begin{align}
    \sum_t \left[ f^t(X^t_i) - f^t(X^t_{i-1}) + f^t(Y^t_i) - f^t(Y^t_{i-1}) \right]
      &= \sum_{t \in \Yes} \alpha^t_i + \sum_{t \in \No}
        \beta^t_i. \label{eqn:fn} 
  \end{align}
  
  We finish by summing the last two bullets above (for both cases) over all the rounds.
  \begin{align}
    \sum_{t=1}^T \left( f^t(OPT^t_{i-1}) - f^t(OPT^t_i) \right)
      &\le \begin{cases}
        \sum_{t \in \No} \alpha^t_i
          &\qquad \text{if } i \in OPT \\
        \sum_{t \in \Yes} \beta^t_i
          &\qquad \text{if } i \not \in OPT
      \end{cases} \nonumber \\
      &\le \max \left(
        \sum_{t \in \No} \alpha^t_i,
        \sum_{t \in \Yes} \beta^t_i
      \right) \label{ieq:opt}
  \end{align}
  
  We must now discuss an important issue before we can proceed with the proof. Suppose that our subroutine $A_i$ is only effective against oblivious adversaries, not adaptive adversaries. We must ensure that its input (namely the sequence $(\alpha^t_i, \beta^t_i)_t$) does not depend on its output. Fortunately, this is the case. In addition to depending on the actual adversary, this sequence depends on the output of subroutines $A_1, \ldots, A_{i-1}$ over all rounds. If the actual adversary is oblivious, then it does not depend on the output of $A_i$. Since they come before $A_i$, none of $A_1, \ldots, A_{i-1}$ depend on $A_i$'s output either (in particular, their inputs do not, so their outputs cannot either)! Put another way, the dependency graph between our subroutines is a directed acyclic graph. If this was not the case, we would have required that they be impervious to adaptive adversaries, in order to handle each other's output. Luckily, we may safely proceed with algorithms that just handle oblivious adversaries. We may now invoke the $\alpha$-regret guarantee for our subroutine $A_i$. Written out completely, the definition of $\alpha$-regret for the \BalanceSubproblem{} states that:
  \begin{align}
    \alpha \cdot
      \expect{ \max
        \left(
          \underbrace{\sum_{t \in \No} \alpha^t_i}_{\CostYes},
          \underbrace{\sum_{t \in \Yes} \beta^t_i}_{\CostNo}
        \right)
      }
      - \expect{ \underbrace{\sum_{t \in \Yes} \frac12 \alpha^t_i + \sum_{t \in \No} \frac12 \beta^t_i}_{\RewardAlg} }
      &\le g(T). \label{ieq:regret}
  \end{align}
  where the expectation is over the random coin flips of $A_i$ and also $A_1, \ldots, A_{i-1}$. We now combine lines~\ref{eqn:fn}-\ref{ieq:regret}.
  \begin{align*}
    \alpha \cdot
      \expect{ \sum_{t=1}^T \left( f^t(OPT^t_{i-1}) - f^t(OPT^t_i) \right) } \\
      - \frac12 \expect{ \sum_t \left[ f^t(X^t_i) - f^t(X^t_{i-1}) + f^t(Y^t_i) - f^t(Y^t_{i-1}) \right] }
      &\le g(T)
  \end{align*}
  This implements our stated plan; the growth of $f^t(X^t_i)$ and/or $f^t(Y^t_i)$ roughly dominates the amount that $f^t(OPT^t_i)$ drops. We now sum over the elements $i \in [n]$.
  \begin{align*}
    \alpha \cdot
      \expect{
        \sum_{t=1}^T \left( \underbrace{f^t(OPT^t_0)}_{=f^t(OPT)}
          - \underbrace{f^t(OPT^t_n)}_{=f^t(ALG^t)} \right) } \\
      - \frac12 \expect{
        \sum_t \left[ \underbrace{f^t(X^t_n)}_{=f^t(ALG^t)}
          - \underbrace{f^t(X^t_0)}_{=f^t(\emptyset)\ge 0}
          + \underbrace{f^t(Y^t_n)}_{=f^t(ALG^t)}
          - \underbrace{f^t(Y^t_0)}_{=f^t([n]) \ge 0} \right] }
      &\le n \cdot g(T)
  \end{align*}
  We finish with some slight rearranging.
  \begin{align*}
    \alpha \cdot \expect{ \sum_{t=1}^T \left( f^t(OPT) - f^t(ALG^t) \right) } - \expect{ \sum_t f^t(ALG^t) } &\le n \cdot g(T) \\
    \alpha \cdot \sum_{t=1}^T f^t(OPT) - (1 + \alpha) \cdot \expect{ \sum_t f^t(ALG^t) } &\le n \cdot g(T) \\
    \frac{\alpha}{1+\alpha} \cdot \sum_{t=1}^T f^t(OPT) - \cdot \expect{ \sum_t f^t(ALG^t) } &\le \frac{1}{1+\alpha} \cdot n \cdot g(T)
  \end{align*}
  Such a subroutine gives Algorithm~\ref{alg:online-usm-framework} the stated $\frac{\alpha}{1+\alpha}$-regret.
\end{proof}

With this reduction in hand, we can make a key observation. We claim that any no-regret algorithm for the (two) experts problem is also a no-$\frac12$-regret algorithm for the \BalanceSubproblem{}. Suppose that an algorithm has $g(T)$ regret for the (two) experts problem and it says yes in rounds $\Yes$ and no in rounds $\No$.
\begin{align*}
  \expect{ \sum_{t=1}^T \alpha^t_i } - \expect{ \sum_{t \in \Yes} \alpha^t_i + \sum_{t \in \No} \beta^t_i } &\le g(T) \\
  \expect{ \sum_{t \in \No} \alpha^t_i } - \expect{ \sum_{t \in \No} \beta^t_i } &\le g(T) \\
  \expect{ \sum_{t=1}^T \beta^t_i } - \expect{ \sum_{t \in \Yes} \alpha^t_i + \sum_{t \in \No} \beta^t_i } &\le g(T) \\
  \expect{ \sum_{t \in \Yes} \beta^t_i } - \expect{ \sum_{t \in \Yes} \alpha^t_i } &\le g(T) \\
  \expect{ \underbrace{\sum_{t \in \No} \alpha^t_i}_{=\CostYes}
              + \underbrace{\sum_{t \in \Yes} \beta^t_i}_{=\CostNo} }
    - \expect{ \underbrace{\sum_{t \in \No} \beta^t_i + \sum_{t \in \Yes} \alpha^t_i}_{=2\RewardAlg} } &\le 2 \cdot g(T) \\
  \frac12 \cdot \expect{ \max (\CostYes, \CostNo) } - \expect{ \RewardAlg } &\le g(T)
\end{align*}
In other words, the algorithm also has $g(T)$ $\frac12$-regret for the \BalanceSubproblem{}, as we claimed earlier. Combining this with Theorem~\ref{thm:balance-subproblem}, we have arrived at the following partial result:
\begin{corollary}\label{cor:one-third}
  When given a subroutine $A$ with $g(T)$ regret for the (two) experts problem, Algorithm~\ref{alg:online-usm-framework} has $O(n \cdot g(T))$ $1/3$-regret.
\end{corollary}

Even without this proof, we might have expected that a claim like
Corollary~\ref{cor:one-third} should be true. After all, over time, a
good algorithm for the two experts problem learns to pick the better
(on average) expert. This corresponds to making an offline greedy
decision, which according to the \cite{BFNS15} analysis is good enough
to get a $1/3$-approximation. However, there are some subtleties that
can occur. For example, the subroutine can sometimes make mistakes,
possibly picking a negative value over a positive one sometimes. The
original \cite{BFNS15} analysis did not need to account for the
possibility of such events, but our proofs implicitly handle them.

\section{An Optimal No-$\frac12$-Regret Algorithm for  \OnlineUSM{}}\label{s:half}
\newcommand{\OptSubroutine}{\textsc{USM Balancer}}

We have now identified a clear goal. In this section, we successfully give a no-regret algorithm for the \BalanceSubproblem{}. Note that this is the optimal value of $\alpha$ in terms of $\alpha$-regret, since Theorem~\ref{thm:balance-subproblem} also transforms inapproximability of \OnlineUSM{} (nothing better than $1/2$) into inapproximability of the \BalanceSubproblem{} (nothing better than $1$). Due to our unusual definition of $\alpha$-regret for the \BalanceSubproblem{}, this was nonobvious. It is perhaps suprising that such a simple algorithm manages to obtain the optimal approximation ratio; the brunt of the work is in the analysis.

\begin{algorithm}
\caption{\OptSubroutine{}}
\label{alg:balancer}

\vspace{.5\baselineskip}

  Initialize $x \leftarrow \frac{1}{2} \sqrt{T}$. \\
  \For{round $t = 1$ to $T$}{
    Compute probability $p^t \leftarrow \frac{x}{\sqrt{T}}$. \\
    Choose the item with probability $p^t$ for round $t$, and receive the point $(\alpha^t, \beta^t)$, \\
    Write $(\alpha^t, \beta^t)$ as the convex combination $c_u (+1, +1) + c_r (+1, -1) + c_\ell (-1, +1)$. \\
    Perform update $x \leftarrow x + (1 - 2p^t) c_u + c_r - c_\ell$. \\
    Cap $x$ back into the interval $[0, \sqrt{T}]$.
  }
\end{algorithm}

Our proposed subroutine is Algorithm~\ref{alg:balancer}.\footnote{As
  presented, our algorithm needs to know the time horizon $T$. This
  dependence can be removed with a standard trick: simply guess $T =
  1$, and double $T$ while restarting the algorithm everytime the
  current guess is violated.}
We defer the proof of its regret to Appendix~\ref{a:half}.

\begin{theorem}\label{thm:balancer}
  \OptSubroutine{} solves the \BalanceSubproblem{} with $O(\sqrt{T})$ regret.
\end{theorem}

\begin{corollary}
  Algorithm~\ref{alg:online-usm-framework} has $O(n\sqrt{T})$ $1/2$-regret when using \OptSubroutine{} as a subroutine.
\end{corollary}

\begin{proof}
  We combine the guarantee about \OptSubroutine{} given by Theorem~\ref{thm:balancer} with the reduction in Theorem~\ref{thm:balance-subproblem}.
\end{proof}

\bibliographystyle{apalike}
\bibliography{usm}

\begin{thebibliography}{}

\bibitem[Ageev and Sviridenko, 1999]{AS99}
Ageev, A.~A. and Sviridenko, M.~I. (1999).
\newblock An 0.828 approximation algorithm for the uncapacitated facility
  location problem.
\newblock {\em Discrete Applied Mathematics}, 93:149–156.

\bibitem[Awerbuch and Kleinberg, 2008]{AK08}
Awerbuch, B. and Kleinberg, R.~D. (2008).
\newblock Online linear optimization and adaptive routing.
\newblock {\em Journal of Computer and System Sciences}, 74(1):97--114.

\bibitem[Bach, 2013]{bach13}
Bach, F. (2013).
\newblock Learning with submodular functions: A convex optimization
  perspective.
\newblock {\em Foundations and Trends in Machine Learning}, 6(2-3):145--373.

\bibitem[Blackwell, 1956]{Blackwell56}
Blackwell, D. (1956).
\newblock An analog of the minimax theorem for vector payoffs.
\newblock {\em Pacific Journal of Mathematics}, 6(1):1--8.

\bibitem[Buchbinder and Feldman, 2016]{BF16}
Buchbinder, N. and Feldman, M. (2016).
\newblock Deterministic algorithms for submodular maximization problems.
\newblock In {\em Proceedings of the Twenty-Seventh Annual ACM-SIAM Symposium
  on Discrete Algorithms}, pages 392--403. Society for Industrial and Applied
  Mathematics.

\bibitem[Buchbinder et~al., 2015a]{BFNS15}
Buchbinder, N., Feldman, M., Naor, J.~S., and Schwartz, R. (2015a).
\newblock A tight linear time (1/2)-approximation for unconstrained submodular
  maximization.
\newblock {\em SIAM Journal on Computing}, 44(5):1384--1402.

\bibitem[Buchbinder et~al., 2015b]{BFS15}
Buchbinder, N., Feldman, M., and Schwartz, R. (2015b).
\newblock Online submodular maximization with preemption.
\newblock In {\em Proceedings of SODA}, pages 1202--–1216.

\bibitem[Cesa-Bianchi and Lugosi, 2006]{CL06}
Cesa-Bianchi, N. and Lugosi, G. (2006).
\newblock {\em Prediction, Learning, and Games}.

\bibitem[Cesa-Bianchi et~al., 2007]{CBMS07}
Cesa-Bianchi, N., Mansour, Y., and Stolz, G. (2007).
\newblock Improved second-order bounds for prediction with expert advice.
\newblock {\em Machine Learning}, 66(2--3):321--352.

\bibitem[Dobzinski and Vondr{\'a}k, 2012]{DV12}
Dobzinski, S. and Vondr{\'a}k, J. (2012).
\newblock From query complexity to computational complexity.
\newblock In {\em Proceedings of the forty-fourth annual ACM symposium on
  Theory of computing}, pages 1107--1116. ACM.

\bibitem[Dudik et~al., 2017]{D+17}
Dudik, M., Haghtalab, N., Luo, H., Scahpire, R.~E., Sygkanis, V., and {Wortman
  Vaughan}, J. (2017).
\newblock Oracle-efficient online learning and auction design.
\newblock In {\em Proceedings of {\em FOCS}}.

\bibitem[Dughmi, 2011]{shaddin}
Dughmi, S. (2011).
\newblock Submodular functions: Extensions, distributions, and algorithms. a
  survey.
\newblock arXiv:0912.0322v4.

\bibitem[Dughmi et~al., 2012]{revsub}
Dughmi, S., Roughgarden, T., and Sundararajan, M. (2012).
\newblock Revenue submodularity.
\newblock {\em Theory of Computing}, 8:95--119.
\newblock Article 5.

\bibitem[Feige et~al., 2011]{FMV11}
Feige, U., Mirrokni, V.~S., and Vondr{\'a}k, J. (2011).
\newblock Maximizing non-monotone submodular functions.
\newblock {\em SIAM Journal on Computing}, 40(4):1133--1153.

\bibitem[Feldman et~al., 2011]{FNS11}
Feldman, M., Naor, J., and Schwartz, R. (2011).
\newblock Nonmonotone submodular maximization via a structural continuous
  greedy algorithm.
\newblock {\em Automata, Languages and Programming}, pages 342--353.

\bibitem[Freund and Schapire, 1997]{FS97}
Freund, Y. and Schapire, R.~E. (1997).
\newblock A decision-theoretic generalization of on-line learning and an
  application to boosting.
\newblock {\em Journal of Computer and System Sciences}, 55(1):119--139.

\bibitem[Fujita et~al., 2013]{FHT13}
Fujita, T., Hatano, K., and Takimoto, E. (2013).
\newblock Combinatorial online prediction via metarounding.
\newblock In {\em Proceedings of {\em ALT}}, pages 68--82.

\bibitem[Goemans and Williamson, 1995]{GW95}
Goemans, M.~X. and Williamson, D.~P. (1995).
\newblock Improved approximation algorithms for maximum cut and satisfiability
  problems using semidefinite programming.
\newblock {\em Journal of the ACM}, 42(6):1115--1145.

\bibitem[Gr{\"o}tschel et~al., 1988]{GLS88}
Gr{\"o}tschel, M., Lov{\'a}sz, L., and Schrijver, A. (1988).
\newblock {\em Geometric algorithms and combinatorial optimization}.
\newblock Springer-Verlag, Berlin.

\bibitem[Guruswami and Khot, 2005]{GK05}
Guruswami, V. and Khot, S. (2005).
\newblock Hardness of {Max 3-SAT} with no mixed clauses.
\newblock In {\em Proceedings of CCC}.

\bibitem[Halperin and Zwick, 2001]{HZ01}
Halperin, E. and Zwick, U. (2001).
\newblock Combinatorial approximation algorithms for the maximum directed cut
  problem.
\newblock In {\em Proceedings of SODA}, pages 1--–7.

\bibitem[Hartline et~al., 2008]{HMS08}
Hartline, J., Mirrokni, V., and Sundararajan, M. (2008).
\newblock Optimal marketing strategies over social networks.
\newblock In {\em Proceedings of the 17th international conference on World
  Wide Web (WWW)}, pages 189--198.

\bibitem[Hazan and Kale, 2012]{HK12}
Hazan, E. and Kale, S. (2012).
\newblock Online submodular minimization.
\newblock {\em Journal of Machine Learning Research}, 13:2903--2922.

\bibitem[Hazan and Koren, 2016]{HK16}
Hazan, E. and Koren, T. (2016).
\newblock The computational power of optimization in online learning.
\newblock In {\em Proceedings of {\em STOC}}.

\bibitem[Iwata et~al., 2001]{IFF01}
Iwata, S., Fleischer, L., and Fujishige, S. (2001).
\newblock A combinatorial strongly polynomial algorithm for minimizing
  submodular functions.
\newblock {\em Journal of the ACM}, 48(4):761--777.

\bibitem[Jegelka and Blimes, 2011]{JB11}
Jegelka, S. and Blimes, J. (2011).
\newblock Online submodular minimization for combinatorial structures.
\newblock In {\em Proceedings of {\em ICML}}.

\bibitem[Kakade et~al., 2009]{KKL09}
Kakade, S.~M., Kalai, A.~T., and Ligett, K. (2009).
\newblock Playing games with approximation algorithms.
\newblock {\em SIAM Journal on Computing}, 39(3):1088--1106.

\bibitem[Kalai and Vempala, 2005]{KV05}
Kalai, A. and Vempala, S. (2005).
\newblock Efficient algorithms for online decision problems.
\newblock {\em Journal of Computer and System Sciences}, 71:291--307.

\bibitem[Milgrom, 2004]{milgrom}
Milgrom, P. (2004).
\newblock {\em Putting Auction Theory to Work}.

\bibitem[{Oveis Gharan} and Vondr{\'a}k, 2011]{GV11}
{Oveis Gharan}, S. and Vondr{\'a}k, J. (2011).
\newblock Submodular maximization by simulated annealing.
\newblock In {\em Proceedings of the twenty-second annual ACM-SIAM symposium on
  Discrete Algorithms}, pages 1098--1116. Society for Industrial and Applied
  Mathematics.

\bibitem[Schrijver, 2000]{S00}
Schrijver, A. (2000).
\newblock A combinatorial algorithm minimizing submodular functions in strongly
  polynomial time.
\newblock {\em Journal of Combinatorial Theory, Series B}, 80(2):346--355.

\bibitem[Schulz and Uhan, 2013]{SU13}
Schulz, A.~S. and Uhan, N.~A. (2013).
\newblock Approximating the least core value and least core of cooperative
  games with supermodular costs.
\newblock {\em Discrete Optimization}, 10:163--180.

\bibitem[Streeter and Golovin, 2009]{SG09}
Streeter, M. and Golovin, D. (2009).
\newblock An online algorithm for maximizing submodular functions.
\newblock In {\em Advances in Neural Information Processing Systems}, pages
  1577--1584.

\bibitem[Vondrak, 2007]{vondrak}
Vondrak, J. (2007).
\newblock {\em Submodularity in Combinatorial Optimization}.
\newblock PhD thesis, Charles University.

\bibitem[Zinkevich, 2003]{Z03}
Zinkevich, M. (2003).
\newblock Online convex programming and generalized infinitesimal gradient
  ascent.
\newblock In {\em Proceedings of {\em ICML}}.

\end{thebibliography}

\appendix

\section{Proof of Theorem~\ref{thm:balancer}}\label{a:half}

In this Appendix, we prove our guarantee for \OptSubroutine{}.

\begin{reminder}{Theorem~\ref{thm:balancer}}
  \OptSubroutine{} solves the \BalanceSubproblem{} with $O(\sqrt{T})$ regret.
\end{reminder}

\begin{proof}
  We need to begin by discussing expectations. The precise inequality we need to
  prove is actually
  \begin{align}
    \E \left[ \max \left( \CostYes, \CostNo \right) - \RewardAlg \right] &\le O(\sqrt{T}). \label{ieq:true}
  \end{align}
  However, we would rather prove this inequality:
  \begin{align}
    \max \left( \E \CostYes, \E \CostNo \right) - \E \RewardAlg &\le O(\sqrt{T}). \label{ieq:proxy}
  \end{align}
  We wish we could use linearity of expectation to make Inequalities~\ref{ieq:true} and
  \ref{ieq:proxy} equivalent. However, the guarantee we want to prove has a max
  inside the expectation, so we cannot freely swap the two. Our first task is
  to show that Inequality~\ref{ieq:proxy} is sufficient.
  
  We now argue that the two random variables may as well have the same
  expectation. Assume without loss of generality that
  $\E \CostYes \ge \E \CostNo$. Let $\CostNo'$ be a random variable equal to
  $\CostNo + \E \CostYes - E \CostNo$, so it has mean $\E \CostYes$ as well.
  This inequality is hence stronger than Inequality~\ref{ieq:true}:
  \[
    \E \left[ \max \left( \CostYes, \CostNo' \right) - \RewardAlg \right] \le O(\sqrt{T}).
  \]
  However, Inequality~\ref{ieq:proxy} is equivalent to:
  \[
    \max \left( \E \CostYes, \E \CostNo' \right) - \E \RewardAlg \le O(\sqrt{T}).
  \]
  We want to prove the following, so that we can add it to the latter to get the
  former:
  \[
    \E \max \left( \CostYes, \CostNo' \right) - \E \CostYes \le O(\sqrt{T}).
  \]
  Let $(x)^+$ denote the positive part of $x$, i.e., $(x)^+ = \max(0, x)$.  We
  prove the following stronger statement:
  \begin{align}
    \E \left( \max \left( \CostYes, \CostNo' \right) - \E \CostYes \right)^+ &\le O(\sqrt{T}) \nonumber \\
    \E \left( \max \left( \CostYes - \E\CostYes, \CostNo' -\E \CostYes\right) \right)^+ &\le O(\sqrt{T}). \label{ieq:newgoal}
  \end{align}
  
  Fortunately, against an oblivious adversary,
  $\CostYes$ is a weighted (all weights are at most a constant) sum of independent
  Bernoulli random variables (the coin tosses we perform each round based on
  $p^t$). They are independent since the adversary must fix a sequence up front,
  to which our online algorithm always chooses the same probabilities $p^t$ for.
  Since variances add over independent variables, this means the variance of
  $\CostYes$ is $O(T)$. Similarly, the variance of $\CostNo$ is $O(T)$ as well
  (although the two are not independent, since they use the same coins). We
  then apply Jensen's inequality to transform our variance bounds into bounds on
  the expected amount variables may exceed their means.
  \begin{align*}
    \E \left[ (\CostYes - \E \CostYes)^2 \right] &\le O(T) &\qquad
    \E \left[ (\CostNo - \E \CostNo)^2 \right] &\le O(T) \\
    \E \left[ \left\lvert \CostYes - \E \CostYes \right\rvert \right] &\le O(\sqrt{T}) &\qquad
    \E \left[ \left\lvert \CostNo - \E \CostNo \right\rvert \right] &\le O(\sqrt{T}) \\
    \E \left[ \left( \CostYes - \E \CostYes \right)^+ \right] &\le O(\sqrt{T}) &\qquad
    \E \left[ \left( \CostNo - \E \CostNo \right)^+ \right] &\le O(\sqrt{T})
  \end{align*}
  Since $\CostNo'$ is just a translated version of $\CostNo$, this guarantee
  holds for $\CostNo'$ as well. This is now good enough to prove
  Inequality~\ref{ieq:newgoal}, because for any two positive numbers
  $x, y$, we know that $\max (x, y) \le x + y$.
  
  We have now finished justifying why Inequality~\ref{ieq:proxy} is sufficient,
  and can proceed to the main proof. Our strategy is as follows. We do not try to
  analyze $\RewardAlg$, $\CostYes$, and $\CostNo$ by themselves. Instead, we
  add the potential functions $\PotentialAlg$, $\PotentialYes$, and
  $\PotentialNo$ to them, respectively. We will show that the algorithm's sum is
  at least as much as the better of the adversary's two sums. Here are our three
  potential functions and their derivatives:
  \begin{itemize}
    \item $\PotentialAlg(x) = \frac{\sqrt{T}}{8} -\frac18 \frac{(2x-\sqrt{T})^2}{\sqrt{T}}$ with derivative $\PotentialAlg'(x) = -\frac12 \frac{(2x - \sqrt{T})}{\sqrt{T}} = \frac12 (1 - 2p^t)$.
    \item $\PotentialYes(x) = \frac12 \frac{(\sqrt{T}-x)^2}{\sqrt{T}}$ with derivative $\PotentialYes'(x) = \frac{(x - \sqrt{T})}{\sqrt{T}} = (p^t-1)$.
    \item $\PotentialNo(x) = \frac12 \frac{x^2}{\sqrt{T}}$ with derivative $\PotentialNo'(x) = \frac{x}{\sqrt{T}} = p^t$.
  \end{itemize}
  
  The potential functions depend on the algorithm's current value for $x$. Since the algorithm maintains $x$ to be in the interval $[0, \sqrt{T}]$, these potential functions always fall in the range $[0, \frac{\sqrt{T}}{2}]$. Since all potentials are bounded in magnitude by $O(\sqrt{T})$, it suffices to prove the following, which we attempt to maintain as an invariant over steps:
  \begin{align}
    \E (\RewardAlg + \PotentialAlg) + O(\sqrt{T})
      \ge \max \left( \E (\CostYes + \PotentialYes), \E (\CostNo + \PotentialNo) \right). \label{ieq:invariant}
  \end{align}
  
  There are only two ways that the algorithm affects rewards, costs, or
  potentials. The first way is that the algorithm may cap $x$ back into the
  interval $[0, \sqrt{T}]$. Since this does not involve interaction with the
  adversary, only the potential functions change in value. However,
  $\PotentialAlg$ is a quadratic which is maximized at $x = \sqrt{T}/2$,
  $\PotentialYes$ is a quadratic which is minimized at $x = \sqrt{T}$, and
  $\PotentialNo$ is a quadratic which is minimized at $x = 0$. Hence, capping
  $x$ only moves it closer to the maximum (resp. minimum) of one of these
  functions, and so only increases (resp. decreases) the function value, in our
  favor. Hence, this process maintains Invariant~\ref{ieq:invariant}.
  
  The second way is that the algorithm chooses the item with probability $p^t$
  and interacts with the adversary. This results in changes to the rewards and
  costs as well as an update to $x$, which changes the potentials. Recall that
  the adversary's possible points are depicted in Figure~\ref{fig:three-moves},
  and that the adversary's point is always a convex combination of three
  extremal choices: right $(+1, -1)$, left $(-1, +1)$ and up $(+1, +1)$.
  
  We need to understand how the potential functions change as $x$ is updated.
  Notice that the update to $x$ never changes it by more than $1$. We observe
  that when $x$ changes by at most $1$, all the potential derivatives change by
  at most $\frac{1}{\sqrt{T}}$. Formally, let there be a constant $\delta$ such
  that $|\delta| \le 1$.
  \begin{align*}
    \PotentialAlg'(x + \delta) - \PotentialAlg'(x)
      &= \left[ -\frac12 \frac{(2(x + \delta) - \sqrt{T})}{\sqrt{T}} \right]
      - \left[ -\frac12 \frac{(2x - \sqrt{T})}{\sqrt{T}} \right] \\
      &= -\frac{\delta}{\sqrt{T}} \\
    \PotentialYes'(x + \delta) - \PotentialYes'(x)
      &= \left[ \frac{(x + \delta - \sqrt{T})}{\sqrt{T}} \right]
      - \left[ \frac{(x - \sqrt{T})}{\sqrt{T}} \right] \\
      &= \frac{\delta}{\sqrt{T}} \\
    \PotentialNo'(x + \delta) - \PotentialNo'(x)
      &= \left[ \frac{x + \delta}{\sqrt{T}} \right]
      - \left[ \frac{x}{\sqrt{T}} \right] \\
      &= \frac{\delta}{\sqrt{T}}
  \end{align*}
  We can now approximate the amount that the potentials themselves change. Let
  $\Potential$ be one of the potential functions, and remember that
  $|\delta| \le 1$.
  \begin{align*}
    \Potential(x + \delta) - \Potential(x)
      &= \int_x^{x + \delta} \Potential'(y) dy \\
      &= \int_x^{x + \delta} \left( \Potential'(x) \pm \frac{1}{\sqrt{T}} \right) dy \\
      &= (x + \delta - x) \left( \Potential'(x) \pm \frac{1}{\sqrt{T}} \right) \\
      &= \delta \cdot \Potential'(x) \pm \frac{1}{\sqrt{T}}
  \end{align*}
  
  Suppose the adversary chooses the extreme point ``right'' $(+1, -1)$. Then $\RewardAlg$ increases by $\frac12 (2p^t-1)$, $\CostYes$ increases by $(1 - p^t)$, and $\CostNo$ increases by $-p^t$. Our algorithm responds by increasing $x$ by $1$, which affects the potentials according to our previous analysis.
  \begin{align*}
    \PotentialAlg(x+1) - \PotentialAlg(x)
      &= 1 \cdot \PotentialAlg'(x) \pm \frac{1}{\sqrt{T}} \\
      &= \frac12 (1 - 2p^t) \pm \frac{1}{\sqrt{T}} \\
    \PotentialYes(x+1) - \PotentialYes(x)
      &= 1 \cdot \PotentialYes'(x) \pm \frac{1}{\sqrt{T}} \\
      &= (p^t - 1) \pm \frac{1}{\sqrt{T}} \\
    \PotentialNo(x+1) - \PotentialNo(x)
      &= 1 \cdot \PotentialNo'(x) \pm \frac{1}{\sqrt{T}} \\
      &= p^t \pm \frac{1}{\sqrt{T}}
  \end{align*}
  In other words, temporarily ignoring our $\pm \frac{1}{\sqrt{T}}$ error bounds, the changes to the potential functions cancel with the changes to the rewards and costs with respect to the sums in Invariant~\ref{ieq:invariant}. By symmetry, the same happens when the adversary chooses the extreme point ``left'' $(-1, +1)$; everything cancels except for the error terms.

  The only remaining extreme point is ``up'' $(+1, +1)$. For this case, $\RewardAlg$ increases by $\frac12$, $\CostYes$ increases by $(1 - p^t)$, and $\CostNo$ increases by $p^t$. Our algorithm responds by changing $x$ by $(1 - 2p^t)$. This again affects the potentials.
  \begin{align*}
    \PotentialAlg(x+1-2p^t) - \PotentialAlg(x)
      &= \frac12 (1 - 2p^t)^2 \pm \frac{1}{\sqrt{T}} \\
    \PotentialYes(x+1-2p^t) - \PotentialYes(x)
      &= (1-2p^t) (p^t - 1) \pm \frac{1}{\sqrt{T}} \\
    \PotentialNo(x+1-2p^t) - \PotentialNo(x)
      &= (1-2p^t) p^t \pm \frac{1}{\sqrt{T}}
  \end{align*}
  The net effect, hiding error terms, is that $\RewardAlg + \PotentialAlg$ increases by $\frac12 (1 + (1-2p^t)^2) \ge \frac12$, while $\CostYes + \PotentialYes$ increases by $2p^t(1-p^t) \le \frac12$ and $\CostNo + \PotentialNo$ also increases by $2p^t(1-p^t) \le \frac12$. Hence for this move we maintain Invariant~\ref{ieq:invariant}, not accounting for error terms.

  We have maintained the invariant for the three extremal moves. However, all other adversary moves are just convex combinations of these three moves, and the algorithm reacts with a convex combination of the appropriate replies. Hence the invariant is maintained for all of the adversary's choice of move.
  
  It remains to briefly discuss the $\pm \frac{1}{\sqrt{T}}$ error we pick up. We pick up this error each round, and there are $T$ total rounds, so the total error regarding our rewards, costs, and potentials is $\pm \sqrt{T}$. We still get Invariant~\ref{ieq:invariant}, but have to increase the constant in the $O(\sqrt{T})$ term by one. Since the invariant was enough to finish the proof, we are now done.
\end{proof}

\section{Adaptive Adversaries}\label{a:adaptive}

When analyzing an online algorithm, we may consider oblivious adversaries or adaptive adversaries. Oblivious adversaries fix the entire input sequence up front, unable to react to the decisions of the online algorithm. On the other hand, adaptive adversaries can choose the next piece of the input to depend on what the algorithm has done so far. For example, for the standard experts problem, the multiplicative weights algorithm works even against adaptive adversaries~\citep{KV05}.

Our techniques work against adaptive adversaries as well. Our framework for \OnlineUSM{} simply deterministically decomposes the problem. Deterministic algorithms are not affected by the oblivious/adaptive swap, because even an oblivious adversary knows what the deterministic online algorithm will do and can hence simulate an adaptive adversary. If we use multiplicative weights as the subroutine for our framework, then we inherit its immunity to adaptive adversaries when producing a no $\frac13$-regret algorithm.

Is our \OptSubroutine{} capable of handling adaptive adversaries as well? It turns out that the answer is yes. We do not need to make any changes to the algorithm, but fixing the proof is a little tricky. We will need to move away from expectations, because the coin flips of our algorithm and adversarial input sequence are now intertwined. We begin by observing that we really managed to prove the following invariant (with some slight rearranging), which is true for any adversarial sequence:
\begin{align*}
  \max \left( \sum_{t=1}^T p^t \beta^t + \PotentialNo, \sum_{t=1}^T (1 - p^t) \alpha^t + \PotentialYes \right) \\
    - \left( \sum_{t=1}^T \frac12 p^t \alpha^t + \sum_{t=1}^T \frac12 (1 - p^t) \beta^t + \PotentialAlg \right)
    &\le O(\sqrt{T})
\end{align*}
and we want to wind up with the following regret guarantee in expectation over the coin flips of the algorithm:
\begin{align*}
  \max \left( \sum_{t \in \Yes} \beta^t, \sum_{t \in \No} \alpha^t \right)
    - \left( \sum_{t \in \Yes} \frac12 \alpha^t + \sum_{t \in \No} \frac12 \beta^t \right)
    &\le O(\sqrt{T}).
\end{align*}

As we have suggestively hinted at by lining up matching terms, the difference between these guarantees boils down to the following question: how much do we expect a sum of Bernoulli variables to differ from their means? The issue we are faced with is that an adaptive adversary may choose the mean of a future variable to depend on the result of a past variable. Nevertheless, we show that even under this condition Bernoulli variables minus their means will have covariance zero (note that the later means are random variables as well):
\begin{lemma}
\label{lem:covariance}
  Consider two random variables $X_1, X_2$ determined by the following process:
  \begin{enumerate}
    \item An adversary selects a probability $p_1 \in [0, 1]$.
    \item $X_1$ is drawn as a Bernoulli variable with mean $p_1$.
    \item The adversary looks at $X_1$ and then selects a probability $p_2 \in [0, 1]$.
    \item $X_2$ is drawn as a Bernoulli varaible with mean $p_2$.
  \end{enumerate}
  Then the covariance of $X_1 - p_1$ and $X_2 - p_2$ is zero.
\end{lemma}

\begin{proof}
  We first use the definition of covariance and simplify, noting that $X_1$ and
  $p_1$ have the same expectation, as do $X_2$ and $p_2$.
  \begin{align*}
    \cov(X_1 - p_1, X_2 - p_2)
      &=
      \E \left[
        \left( (X_1 - p_1) - \E \left[ X_1 - p_1 \right] \right)
        \left( (X_2 - p_2) - \E \left[ X_2 - p_2 \right] \right)
      \right] \\
      &=
      \E \left[
        \left( X_1 - p_1 \right)
        \left( X_2 - p_2 \right)
      \right]
  \end{align*}
  Next, we note that the expected value of $X_2 - p_2$ is always zero even if we
  condition on the values of $X_1$ and $p_1$.
  \begin{align*}
    \E \left[
      \left( X_1 - p_1 \right)
      \left( X_2 - p_2 \right)
      \mid
      X_1, p_1
    \right] &= 0 \\
    \E \left[
      \left( X_1 - p_1 \right)
      \left( X_2 - p_2 \right)
    \right] &= 0 \\
    \cov(X_1 - p_1, X_2 - p_2) &= 0
  \end{align*}
  This completes the proof.
\end{proof}

As a result of Lemma~\ref{lem:covariance} we can bound the difference between
matching sums of our invariant and desired regret guarantee. For example,
consider the two sums $\sum_{t=1}^T p^t \beta^t$ and
$\sum_{t \in \Yes} \beta^t$. What is the expected difference between them?
We know that the variance of a single term in the sum is $O(1)$, since each term
is the difference between a Bernoulli random variable and its mean, times a
value $\beta^t$ which is at most one. By Lemma~\ref{lem:covariance}, any pair
of differences has covariance zero. Hence the overall variance between these
sums is the desired $O(T)$. We again apply Jensen's to turn a variance
bound into a bound on the expected deviation.
\begin{align*}
  \E \left[
    \left(
      \sum_{t \in \Yes} \beta^t - \sum_{t=1}^T p^t \beta^t
    \right)^2
  \right] &\le O(T) \\
  \E \left[
    \left \lvert
      \sum_{t \in \Yes} \beta^t - \sum_{t=1}^T p^t \beta^t
    \right \rvert
  \right] &\le O(\sqrt{T}) \\
  \E \left[
    \left(
      \sum_{t \in \Yes} \beta^t - \sum_{t=1}^T p^t \beta^t
    \right)^+
  \right] &\le O(\sqrt{T})
\end{align*}

However, there was nothing special about this pair of sums, so we get similar
inequalities for the other three pairs of sums.
\begin{align*}
  \E \left[
    \left(
      \sum_{t \in \No} \alpha^t - \sum_{t=1}^T (1-p^t) \alpha^t
    \right)^+
  \right] &\le O(\sqrt{T}) \\
  \E \left[
    \left(
      - \sum_{t \in \Yes} \frac12 \alpha^t + \sum_{t=1}^T \frac12 p^t \alpha^t
    \right)^+
  \right] &\le O(\sqrt{T}) \\
  \E \left[
    \left(
      - \sum_{t \in \No} \frac12 \beta^t + \sum_{t=1}^T \frac12 (1-p^t) \beta^t
    \right)^+
  \right] &\le O(\sqrt{T})
\end{align*}

We conclude by again noting that for any positive numbers $x, y$,
$\max(x, y) \le x + y$, which lets us swap expectation and max as before. The
final step is noting that by construction, potentials are always $O(\sqrt{T})$
in magnitude.

\end{document}